\documentclass{article}
\usepackage{graphicx}

\PassOptionsToPackage{numbers, compress}{natbib}
\usepackage[final]{nips_2016}

\usepackage[utf8]{inputenc} 
\usepackage[T1]{fontenc}    
\usepackage{hyperref}       
\usepackage{url}            
\usepackage{booktabs}       
\usepackage{amsfonts}       
\usepackage{nicefrac}       
\usepackage{microtype}      
\usepackage{amsthm}

\newtheorem{theorem}{Theorem}[section]

\newtheorem{lemma}[theorem]{Lemma}

\newcommand{\cN}{\mathcal{N}}
\newcommand{\cG}{\mathcal{G}}
\newcommand{\cO}{\mathcal{O}}

\title{Sparse Diffusion-Convolutional Neural Networks}

\author{
  James Atwood\thanks{Now at Google Brain.} \\
  UMass Amherst CICS \\
  Amherst, MA \\
  \texttt{jatwood@cs.umass.edu} 
  \And
  Siddharth Pal \\
  Raytheon BBN Technologies \\
  Cambridge, MA \\
  \texttt{siddharth.pal@raytheon.com} \\
  \And
  Don Towsley \\
  UMass Amherst CICS \\
  Amherst, MA \\
  \texttt{towsley@cs.umass.edu} 
  \And
  Ananthram Swami \\
  U.S. Army Research Lab \\
  Adelphi, MD \\
  \texttt{ananthram.swami.civ@mail.mil}
}

\begin{document}

\maketitle

\let\thefootnote\relax\footnotetext{J. Atwood and S. Pal contributed equally to this work.}

\let\thefootnote\relax\footnotetext{This work was supported in part by Army Research Office Contract W911NF-12-1-0385 and ARL Cooperative Agreement  W911NF-09-2-0053 (the ARL Network Science CTA). This document does not contain technology or technical data controlled under either the U.S. International Traffic in Arms Regulations or the U.S. Export Administration Regulations.}

\begin{abstract}
The predictive power and overall computational efficiency of Diffusion-convolutional neural networks make them an attractive choice for node classification tasks.  However, a naive dense-tensor-based implementation of DCNNs leads to $\mathcal{O}(N^2)$ memory complexity which is prohibitive for large graphs.  In this paper, we introduce a simple method for thresholding input graphs that provably reduces memory requirements of DCNNs to $\mathcal{O}(N)$ (i.e. linear in the number of nodes in the input)  without significantly affecting predictive performance.
\end{abstract}

\section{Introduction}
There has been much recent interest in adapting models and techniques from deep learning to the domain of graph-structured data~\cite{NIPS_atwood,Bruna_2013,NIPS_defferrard,Henaff_2015,ICLR_kipf,Niepert_2016}. Proposed by Atwood and Towsley~\cite{NIPS_atwood}, Diffusion-convolutional neural networks (DCNNs) approach the
problem by learning \lq filters'  that summarize local information in a graph via a diffusion process.  These filters have been observed to provide an effective basis for node classification.

The DCNNs have been shown to possess attractive qualities like obtaining a latent representation for graphical data that is invariant under isomorphism, and utilizing tensor operations that can be efficiently implemented on the GPU. Nevertheless, as was remarked in \cite{NIPS_atwood}, when implemented using dense tensor operations, DCNNs have a $\mathcal{O}(N^2)$ memory complexity, which could get prohibitively large for massive graphs with millions or billions of nodes.

In an effort to improve the memory complexity of the DCNN technique, we investigate two approaches of thresholding the diffusion process --  a pre-thresholding technique that thresholds the transition matrix itself, and a post-thresholding technique that enforces sparsity on the power series on the transition matrix. We show that pre-thresholding the transition matrix provides provably linear ($\mathcal{O}(N)$) memory requirements while the model's predictive performance remains unhampered for small to moderate thresholding values ($\rho \leq 0.1$). On the other hand, the post-thresholding technique did not offer any gains in memory complexity. This result suggests that pre-thresholded sparse DCNNs (sDCCNs) are suitable models for large graphical datasets.

\section{Model}

We study node classification on a single graph, say $\cG = (V, E)$, with $V$ being the vertex or node set, and $E$ being the set of edges.
No constraints are imposed on the graph $\cG$; the graph can be weighted or unweighted, directed or undirected. Each vertex is assumed to be associated with $F$ features, leading to the graph being described by an $N \times F$ design matrix, $X$, and an $N \times N$ adjacency matrix, with $N=|V|$ being the number of vertices. In DCNN, we compute a degree-normalized transition matrix $P$ that gives the probability of moving from one node to another in a single step. However, in a sparse implementation of DCNN (sDCNN), rather than using the transition matrix directly, we remove edges with probabilities below a threshold in order to both improve memory complexity and regularize the graph structure.

Assume the nodes are associated with labels, i.e., each node $i$ in $V$ has a label $y_i$ in $Y$. Given a set of labelled nodes in a graph, the node classification task is to find labels for unlabeled nodes. Note that, while in this work we focus on node classification tasks, this framework can be easily extended to graph classification tasks where graphs have labels associated with them rather than individual nodes~\cite{NIPS_atwood}.

Next, we describe the DCNN framework in greater detail. The neural network takes the graph $\cG$ and the design matrix $X$ as input, and returns a hard prediction for $Y$ or a conditional distribution $\mathbb{P}(Y|X)$ for unlabelled nodes. Each node is transformed to a diffusion-convolutional representation, which is an $H \times F$ real matrix defined by $H$ hops of graph diffusion over $F$ features. The core operation of a DCNN is a mapping from nodes and their features to the results of a diffusion process that begins at that node. The node class label is finally obtained by integrating the result of the diffusion process over the graph through a fully connected layer, thus combining the structural and feature information in the graph data. In sDCNN, the diffusion process itself is thresholded, to reduce the computational complexity of the diffusion process over a large graph. 

\begin{figure}
    \centering
    \includegraphics[scale=0.5]{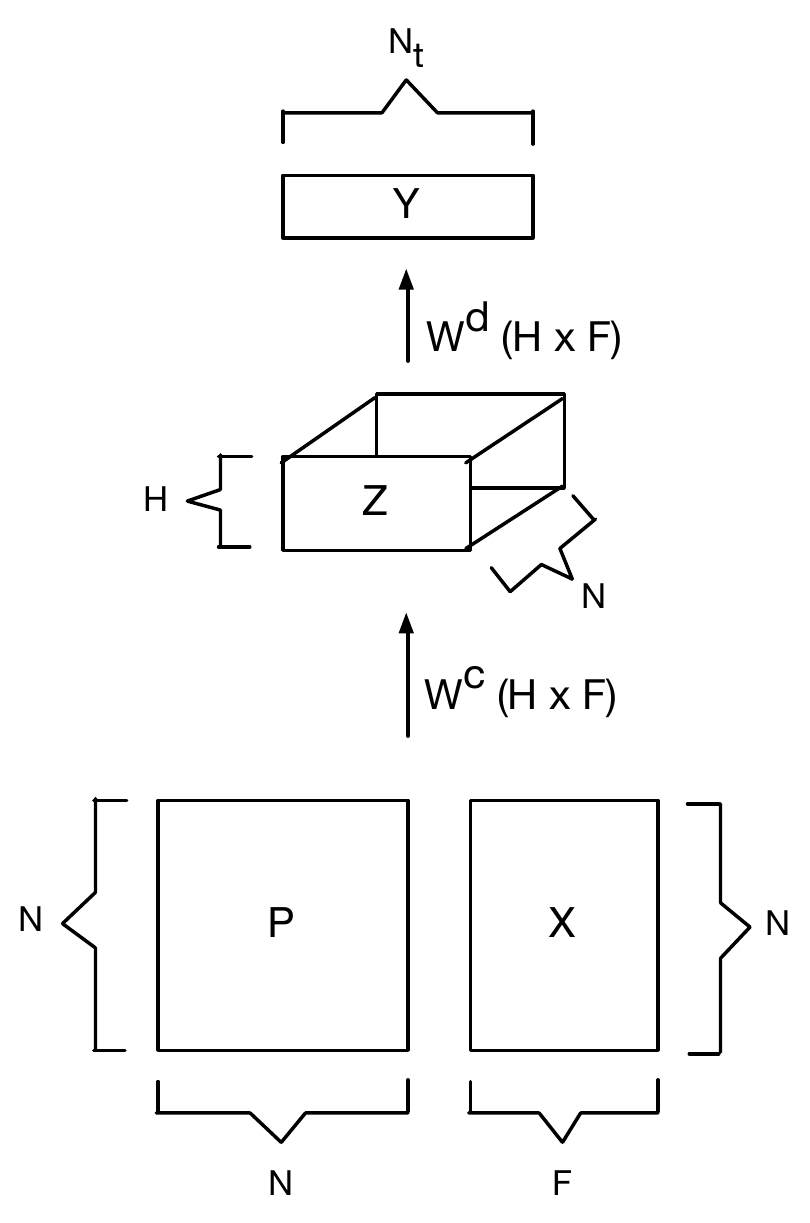}
    \caption{DCNN schematic diagram for node classification tasks}
    \label{fig:DCNN_model}
\end{figure}

\textbf{DCNN with no thresholding} Consider a node classification task 
where a label $Y$ is associated with each node in a graph. Let $P ^*$ be 
an $N \times H \times N$ tensor containing the power series of the transition matrix $P$. The probability of reaching node $l$ from node $i$ through $j$ hops is captured by $P^j_{il}$, or equivalently 
by $P ^* _{ijl}$. The diffusion-convolutional activation $Z_{ijk}$ for node $i$, hop $j$ and feature $k$ of graph $\cG$ is given by 
\begin{equation}
    Z_{ijk} = f 
\left(
W_{jk} ^c \cdot \sum _{l =1} ^N P_{ijl} ^* X_{lk}
\right)
\end{equation}
where $\{ W ^c_{jk}, j=1,2,\ldots, H, \ k=1,2,\ldots,F\}$ are the learned weights of the diffusion-convolutional layer, and $f$ is the activation function. Briefly, the weights determine the effect of neighboring nodes' features on the class label of a particular node. The activations can be expressed more concisely using tensor notation as 
\begin{equation}
    Z = f(W^c \odot P^* X),
\end{equation}
where the $\odot$ operator represents element-wise multiplication. Observe that the model only entails $\cO(H \times F)$ parameters, making the size of the latent diffusion-convolutional representation independent of the size of the input. 

The output from the diffusion-convolutional layer connects to the output layer with $N_t$ neurons through a fully connected layer. A hard prediction for $Y$, denoted $\hat{Y}$, can be obtained by taking the maximum activation, as follows
\begin{equation}
    \hat{Y} = \arg \max \left( f \left( W ^d \odot Z \right) \right),
\end{equation}
whereas, a conditional probability distribution $\mathbb{P}(Y | X)$ can be obtained by applying the softmax function
\begin{equation}
    \mathbb{P} (Y | X) = \mbox{softmax} \left( f \left( W ^d \odot Z \right) \right). 
\end{equation}

Since DCNNs require computation and storage of tensors representing the power series of the transition matrices, it is costly in terms of computational resources. In this work, we investigate methods to enforce sparsity in such tensors, and consequently reduce the utilization of memory. In what follows, we describe two thresholding methods for enforcing sparsity. 

\textbf{Pre-thresholding} Through this technique the transition matrix is first thresholded, and then a power series of the thresholded transition matrix is computed. 
For a threshold value $\sigma$, the pre-thresholded activation $Z_{ijk}^{pre}$ for 
node $i$, hop $j$, and feature $k$ is given by 
\[
Z_{ijk} ^{pre} = f 
\left(
W_{jk} ^c \cdot \sum _{l =1} ^N \bar{P}_{ijl} ^* X_{lk}
\right)
\]
where 
\[
\bar{P}_{il} = 1[P_{il}\geq \sigma] P_{il}
\]
is the thresholded transition matrix, and 
\[
\bar{P}_{ijl} ^*=\bar{P}^j_{il}.
\]
Note that for an event $\mathcal{E}$, $1[\mathcal{E}]=1$, if the event is true, and $0$ otherwise.

\textbf{Post-thresholding} This thresholding method enforces sparsity on the power series of the transition matrix $P$.
For a threshold value $\rho$, the post-thresholded activation $Z_{ijk}^{post}$ for 
node $i$, hop $j$, and feature $k$ is given by
\[
Z_{ijk} ^{post} = f 
\left(
W_{jk} ^c \cdot \sum _{l =1} ^N 1[P_{ijl} ^* \geq \rho] P_{ijl} ^* X_{lk}
\right).
\]

Qualitatively, pre-thresholding only considers strong ties within a particular node's neighborhood with all the intermediary ties being sufficiently strong, whereas, post-thresholding looks at the entire neighborhood of a node, and chooses the strong ties, allowing long hop ties to be chosen, potentially passing through multiple weak ties.
For threshold parameter $\rho$ or $\sigma$ set to zero, all the ties are considered, and we obtain the DCNN setting in this limit. 
On the other hand, when the threshold parameter is set to the maximum value of one, only the relevant node's feature value is considered along with its neighboring node only if the node in question has one neighbor. This is qualitatively close to a logistic regression setting, even though not exactly the same.

\section{Complexity results}

\begin{lemma}
For $H>1$, the memory complexity for the DCNN method is $\cO(N^2+NFH)$.
For $H=1$, the memory complexity is $\cO(|E|+NF)$.
\end{lemma}
\begin{proof}
An efficient way to store the power series would be to store the product of the power of transition matrices with the design matrix. However the intermediate powers of the 
transition matrices need to be stored, which requires $\cO(N^2)$ memory. Storing the 
product of the power series tensor with the design matrix requires $\cO(NFH)$ memory. 
Therefore, the upper bound on the memory usage is $\cO(N^2+NFH)$. 

However if $H=1$, the transition matrix can be represented in $\cO(|E|)$, thereby getting rid of the $\cO(N^2)$ memory requirement.
\end{proof}

\begin{lemma}
For $H>1$ and a fixed threshold $\sigma$, memory complexity under the pre-thresholding technique is 
$\cO \left( \min \left( N\cdot  \frac{1}{\sigma ^{H}} , N^2 \right) + NFH \right)$. For $H=1$, the memory complexity is 
$\cO \left( \min \left( N\cdot  \frac{1}{\sigma ^{H}} , |E|, N^2 \right) + NFH \right)$
\label{lemma:Pre_thresholding_memory}
\end{lemma}

\begin{proof}
We argue that a sparse representation of the power series tensor
product with the design matrix occupies 
$\cO \left( \min \left( N\cdot \frac{1}{\sigma ^{H}}, N^2 
\right) + NFH \right)$ memory in an inductive manner.
For node $i$ in $V$ and hop $h$, we define the set of nodes in the $h$-hop neighborhood that influence node $i$ through pre-thresholding as
\[
\cN _i ^h = \{ l \ | \ \bar{P} ^{h}_{il} > 0 \}.
\]

For $j=0$, $\bar{P} ^j$ is the identity matrix $I_n$, which is sparse with exactly $N$ non-zero entries.

For $j=1$, $\bar{P} ^j$ is the thresholded transition matrix. In the pre-thresholding operation, transition probabilities from node $i$ that are less than $\sigma$ are set to zero.
For a particular node $i$, the set of nodes that
have 1-hop transition probabilities greater than or equal to $\sigma$, 
is exactly $\cN ^1_i$.
Since, $\sum _{l \in V} P_{il} = 1$ and $\sum _{l \in V} \bar{P}_{il} \leq 1$, we must have 
$|\cN ^1 _i| \leq  \min \left(\frac{1}{\sigma}, N \right)$. Therefore, $\bar{P}$ can have at most 
$\min \left( \frac{N}{\sigma}, N^2 \right)$ entries.

For $j=h$, $\bar{P} ^h$ is the thresholded transition matrix raised to the $h^{th}$ power. Suppose the sparse representation is such that 
$\{ \bar{P} _{ihl}, l \in V \}$ has only 
$ \min \left( \frac{1}{\sigma^h} , N \right) $ non-zero entries for each $i$ in $V$, implying that $\bar{P} ^h$ have $\cO( \min \left( N \cdot \frac{1}{\sigma^h}, N^2 \right) )$ entries\\
The final step is to prove that for $j=h+1$, $\bar{P} ^{j}$ has  
$\cO \left( \min \left( N \cdot \frac{1}{\sigma^{h+1}}, N^2 \right) \right)$ entries. For $i$ in V, let
\[
\cN _i ^{h+1} = \{ l \ | \ \bar{P} ^{h+1}_{il} > 0 \}.
\]
By assumption, $|\cN _i ^h| \leq \min \left( \frac{1}{\sigma ^h}, N \right)$. Observe that 
$\cN _i^{h+1}$ is the set $\{ m \ | \ j \in \cN _i ^{h}, m \in \cN _j ^1 \}$. 
Since $| \cN _i ^h | \leq \min \left( \frac{1}{\sigma ^h} , N \right) $ and $|\cN_j ^1| \leq \min \left( \frac{1}{\sigma} , N \right)$
for all $j$ that are in $\cN_i ^h$, we have the bound $| \cN _i ^{h+1} | \leq \min \left( \frac{1}{\sigma ^{h+1}}, N \right) $. Thus, we have proved that 
$\bar{P} ^{h+1}$ has  $\cO \left( \min \left( N \cdot \frac{1}{\sigma^{h+1}}, N^2 \right) \right)$ non-zero entries. 
Thus by induction, the costliest operation is computing $\bar{P} ^H$ which requires $\cO \left( \min \left( N \cdot \frac{1}{\sigma ^{H}}, N^2 \right) \right)$ memory, and the result follows. 
\end{proof}

\begin{lemma}
For $H>1$ and a fixed threshold $\rho$, memory complexity under the post-thresholding technique is 
$\cO \left( N^2 + NFH \right)$. For $H=1$, the memory complexity is $\cO(|E|+NF)$.
\end{lemma}
\begin{proof}
Even though the post-thresholded  power series tensor can be proven to have $\cO \left( \frac{NH}{\rho} \right)$, in a manner similar to that of Lemma \ref{lemma:Pre_thresholding_memory}, the intermediate powers of the dense transition matrix still have to be computed. This requires $\cO(N^2)$ memory, and therefore, no improvement in memory utilization is obtained.
\end{proof}

Assume $H>1$; To obtain the computational complexity of the DCNN method, we observe that two $N \times N$ matrices are multiplied $H$ times, and a $N \times N$ matrix needs to be multiplied with another $N \times F$ matrix, $H$ times. Two $N \times N$ matrices can be multiplied in $\cO(N^{2.38} )$ complexity using efficient matrix multiplication algorithms for square matrices~\cite{coppersmith}. The product between the transition matrix and the design matrix can be performed in $\cO(N ^2 F)$ operations, so the overall complexity is $\cO(N^{2.38} H)$. The computational 
complexity of the post-thresholded sDCNN is also $\cO(N^{2.38} H)$, because the dense power series tensor needs to be computed. 

The pre-thresholded DCNN achieves an improvement in the computational complexity because the power series tensor is computed by multiplying two sparse $N \times N$ matrices. 
The costliest operation is computing $\bar{P} ^H$ which is obtained by multiplying $\bar{P} ^{H-1}$, a sparse matrix
with at most $\frac{N}{\sigma ^H}$ non-zero entries, and $\bar{P}$, another sparse matrix with at most $\frac{N}{\sigma}$ non-zero entries. Using efficient sparse matrix multiplication methods~\cite{LeGall,yuster_zwick}, if 
the condition $\frac{1}{\sigma ^H} < N ^{0.14}$ holds, then the computational complexity of the sparse method is $\cO(N^{2+o(1)})$.

Therefore, pre-thresholded sDCNN achieves $\cO(N)$ memory complexity and $\cO(N^2)$ computational complexity, a significant improvement over DCNN, which requires $\cO(N^2)$ memory and $\cO(N^{2.38})$ computational complexity. 
However, post-thresholded sDCNN still requires the same memory and computational complexity as that of DCNN. Therefore, going forward we will simply be considering pre-thresholded sDCNN, although post-thresholding could be thought of a way of regularizing the DCNN method.

\section{Experiments}
In this section we explore how thresholding affects both the density of transient diffusion kernel and the performance of DCNNs.

\subsection{Effect of Thresholding on Density}
\begin{figure}[h!]
    \centering
    \includegraphics[scale=0.6]{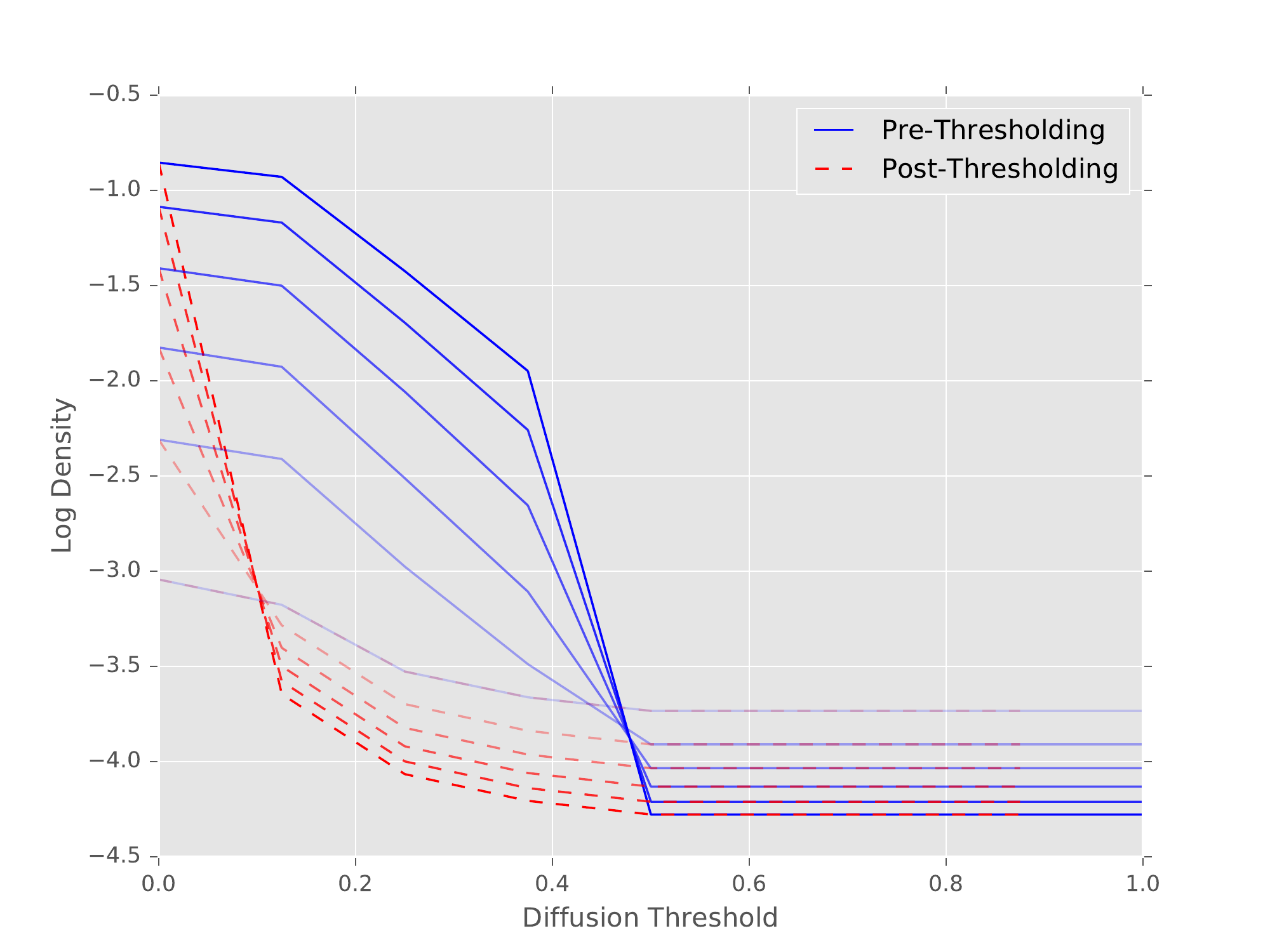}
    \caption{The effect of diffusion threshold on transient diffusion kernel density under pre- and post-thresholding strategies on the Cora dataset.  The density is given by the proportion of non-zero entries in the transient diffusion kernel $I, P, P^2, ..., P^H$ and is plotted on a log scale.    Lighter lines indicate smaller diffusion kernels.}
    \label{fig:diffusionDensity}
\end{figure}

Figure \ref{fig:diffusionDensity} shows the results of applying the two thresholding strategies to the Cora dataset. Observe that, both thresholding techniques show a decrease in diffusion kernel density as the threshold is increased. However the decrease is more gradual for the pre-thresholding method, due to the fact that, transition probabilities reach low values for greater number of hops, which when post-thresholded lead to low densities for relatively slight increase in the diffusion threshold. On the other hand, the pre-thresholding method is better behaved, with the kernel density decreasing in a more gradual fashion. 

The darker lines corresponding to larger diffusion kernels obtained through greater number of hops, $H$, have higher diffusion kernel density for low diffusion threshold. However, as the diffusion threshold is increased, the darker lines cross over the lighter lines around $0.5$, for the pre-thresholding method. The justification for this phenomenon is that as the diffusion threshold is increased to $0.5$, only the contribution of the identity matrix remains, and the larger diffusion kernels therefore show lower density. A similar phenomenon occurs for the pre-thresholding technique, except that the crossover region occurs much earlier. Although we show only the behavior of the Cora dataset, the behavior should hold for other datasets as well.

\subsection{Effect of Thresholding on Performance}
\begin{figure}[h!]
    \centering
    \includegraphics[scale=0.45]{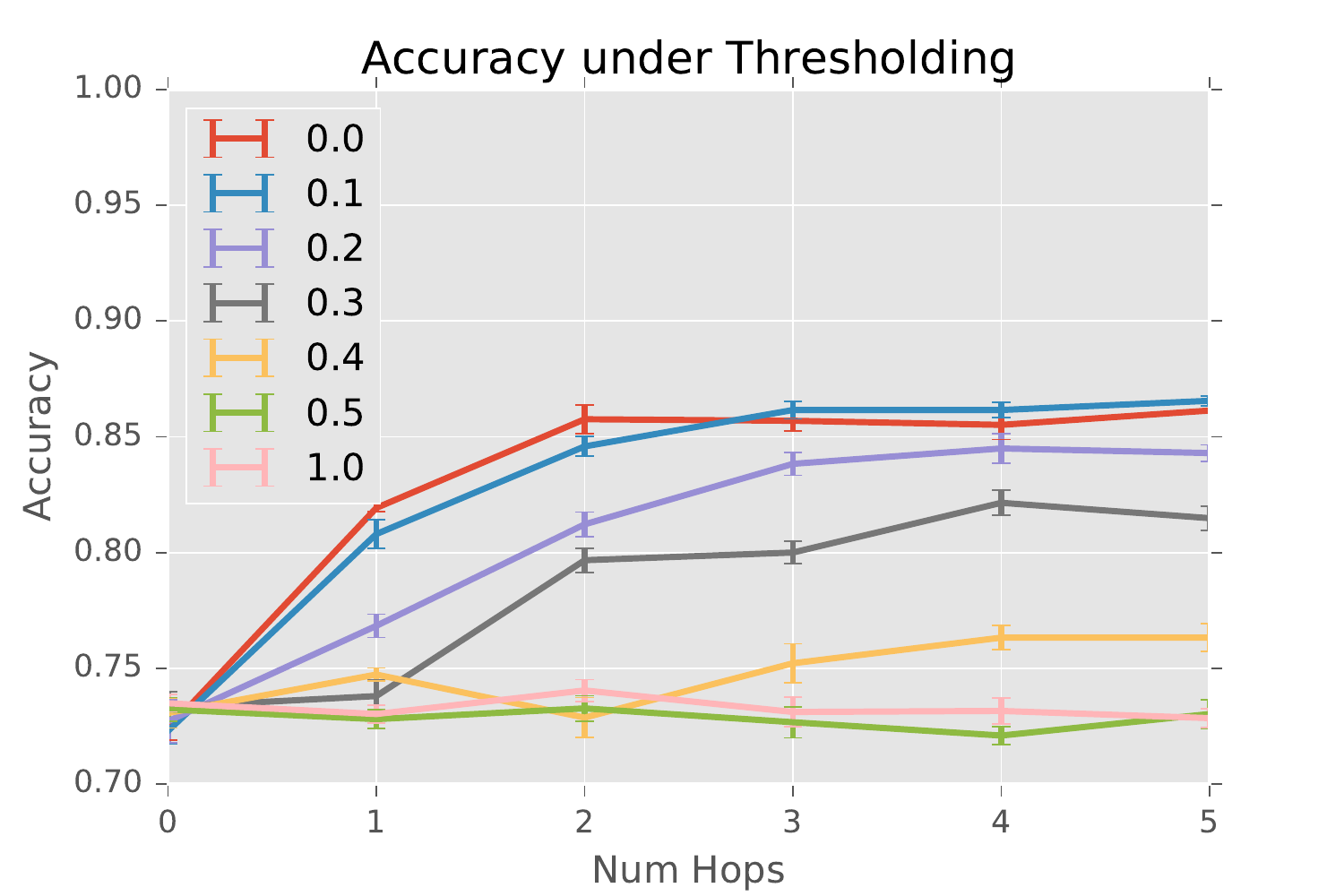}
    \includegraphics[scale=0.45]{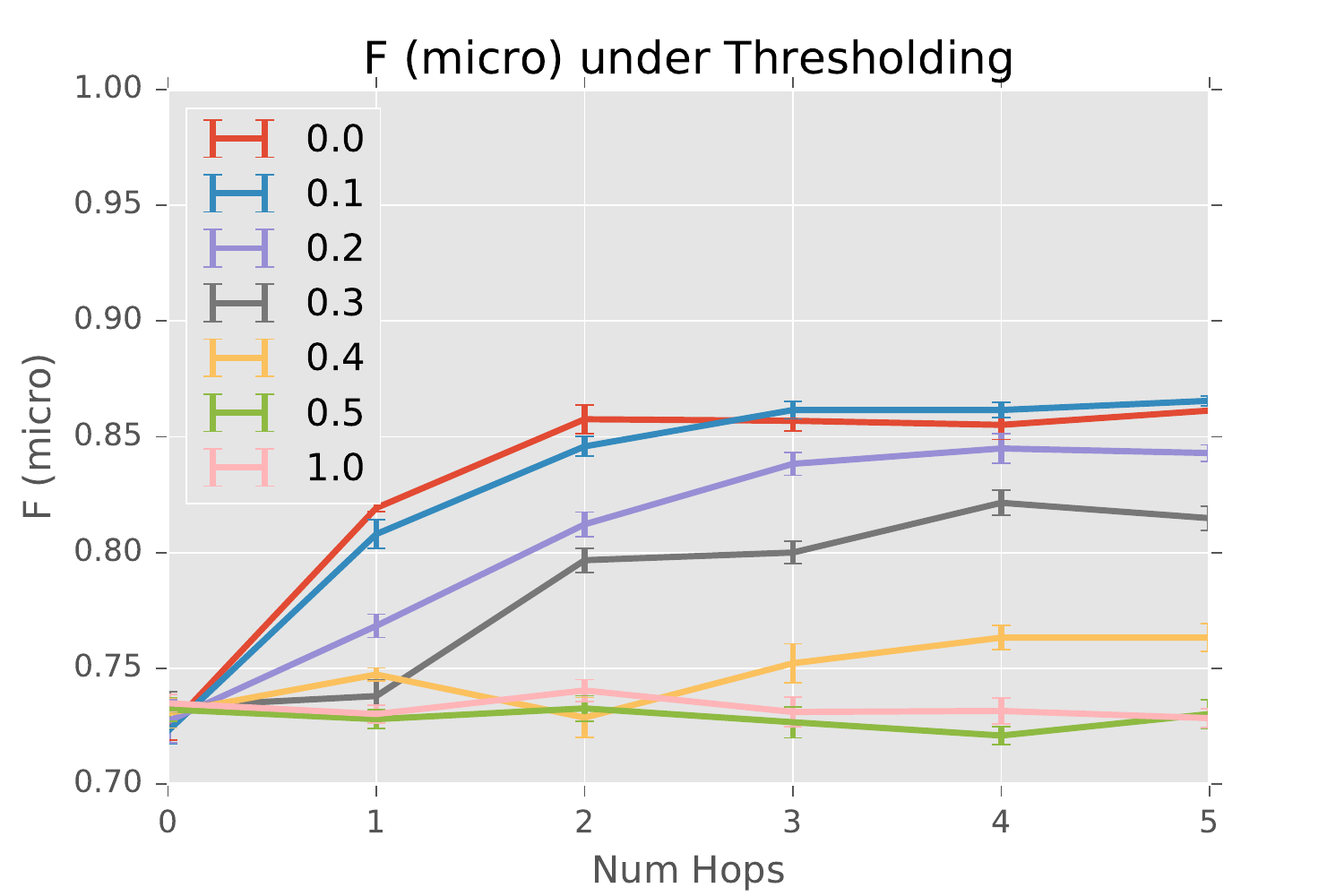}
    \caption{The effect of pre-thresholding on classification performance.  Small thresholds ($\rho \leq 0.1$) have no significant effect on classification performance, while larger thresholds ($\rho \geq 0.5$) remove the benefit of including neighborhood information.}
    \label{fig:thresholdingPerformance}
\end{figure}
Figure \ref{fig:thresholdingPerformance} shows the effect of thresholding on DCNN performance.  Observe that, for both thresholding techniques, small-to-moderate thresholding values ($\rho \leq 0.1$) have no significant effect on performance, although performance degrades for larger thresholds.  This suggests that applying a small threshold when computing the transient diffusion kernel is an effective means of scaling DCNNs to larger graphs.

However, it should be noted that for moderate thresholds $0.1 < \rho < 0.5$, the performance begins to decay.  Eventually, when the threshold reaches the reciprocal of the maximum degree $\rho \geq 0.5$, the benefit of including neighborhood information vanishes entirely because no edges are left in the graph.

\section{Related Work}

Neural networks for graphs were introduced by Gori et al.~\cite{Gori_2005} and followed by Scarselli et al.~\cite{Scarselli_2009}, in departure from the traditional approach of transforming the graph into a simpler representation which could then be tackled by conventional machine learning algorithms. Both the works used recursive neural networks for processing graph data, requiring repeated application of contraction maps until node representations reach a stable state. Bruna et al.~\cite{Bruna_2013} proposed two generalizations of CNNs to signals defined on general domains; one based upon a hierarchical clustering of the domain, and another based on the spectrum of the graph Laplacian. This was followed by Henaff et al.~\cite{Henaff_2015}, which used these techniques to address a setting where the graph structure is not known a priori, and needs to be inferred. However, the parametrization of CNNs developed in \cite{Bruna_2013,Henaff_2015} are dependent on the input graph size, while that of DCNNs or sDCNNs are not, making the technique transferable, i.e., a DCNN or sDCNN learned on one graph can be applied to another. Niepert et al.~\cite{Niepert_2016} proposed a CNN approach which extracts locally connected regions of the input graph, requiring the definition of node ordering as a pre-processing step.

Atwood and Towsley~\cite{NIPS_atwood} had remarked that the DCNN technique uses $\cO(N^2)$ memory. It is worth noting that the sparse implementation of DCNN yields an order improvement in the memory complexity.  This will offer a performance parity with more recent techniques that report $\cO(E)$ memory complexity, like the graph convolutional network method by Kipf et al.~\cite{ICLR_kipf} and the localized spectral filtering method of Defferard et al.~\cite{NIPS_defferrard} when the (unthresholded) input graph is sparse ($\cO(E)$ = $\cO(N)$), which is the case for many real-world datasets of interest, and a performance improvement when graphs are dense  ($\cO(E)$ > $\cO(N)$).

\section{Conclusion}
We have shown that, by applying a simple thresholding technique, we can reduce the computational complexity of diffusion-convolutional neural networks to $\cO(N^2)$ and the memory complexity to $\cO(N)$.  This is achieved without significantly affecting the predictive performance of the model.

\end{document}